%% file: main.tex
\documentclass{article}

\usepackage{arxiv}

\usepackage[utf8]{inputenc} 
\usepackage[T1]{fontenc}    
\usepackage{hyperref}       
\usepackage{url}            
\usepackage{booktabs}       
\usepackage{amsfonts}       
\usepackage{nicefrac}       
\usepackage{microtype}      
\usepackage{lipsum}		
\usepackage{graphicx}
\usepackage{natbib}
\usepackage{doi}
\usepackage{amsmath}
\usepackage{amsthm}
\usepackage{amssymb}
\usepackage{todonotes}
\usepackage{verbatim}
\usepackage{subcaption}

\usepackage{algorithm}
\usepackage[switch]{lineno}
\usepackage[noend]{algpseudocode}

\usepackage{tikz}
\usetikzlibrary{arrows.meta}

\newcommand{\AS}{\mathcal{F}}
\newcommand{\A}{\mathcal{A}}
\newcommand{\C}{\mathcal{C}}

\newcommand{\Att}{Att}

\newcommand{\VPE}{\mathcal{EX}}

\newcommand{\out}{\textbf{out}}
\newcommand{\lin}{\textbf{in}}
\newcommand{\undec}{\textbf{undec}}
\newcommand{\RANK}{\mathcal{RANK}}

\newtheorem{example}{Example}

\newtheorem{corollary}{Corollary}
\newtheorem{definition}{Definition}
\newtheorem{proposition}{Proposition}

\title{On the Existence of an Inverse Solution for Preference-Based Reductions in Argumentation}

\author{
Alessio Zaninotto$^1$\hfill
Bruno Yun$^2$\hfill
Nir Oren$^3$\hfill
Srdjan Vesic$^1$\\
$^1$CRIL - Université d'Artois, CNRS\\
$^2$Universite Claude Bernard Lyon 1, CNRS, Ecole Centrale de Lyon, INSA Lyon,\\ Université Lumière Lyon 2, LIRIS, UMR5205, 69622 Villeurbanne, France\\
$^3$University of Aberdeen\\
zaninotto@cril.fr,
bruno.yun@univ-lyon1.fr,
n.oren@abdn.ac.uk,
vesic@cril.fr
}

\begin{document}

\maketitle

\begin{abstract}
\input{abstract}
\end{abstract}

\section{Introduction}

\input{introduction}

\section{Background}

\input{background}

\section{Studying Preference-Based Reductions}

\input{contribution}

\section{Related and Future Work}
\label{Section: related works}
\input{related_work}

\section{Conclusions}
\label{Section:conclusions}

\input{conclusion}

\bibliographystyle{abbrvnat}
\bibliography{biblio}

\end{document}

%% file: abstract.tex
Preference-based argumentation frameworks (PAFs) extend Dung's approach to abstract argumentation (AAFs) by encoding preferences over arguments. Such preferences control the transformation of attacks into defeats, and different approaches to doing so result in different reductions from a PAF to an AAF.
In this paper we consider a PAF inverse problem which takes an argumentation graph, a labelling and a semantics as an input, and outputs a ``yes" or ``no" as to whether there is a preference relation between the arguments which can yield the desired labelling. This inverse problem has applications in areas including preference elicitation and explainability. We consider this problem in the context of the four most widely-used preference based reductions under the complete semantics. We show that in most cases, the problem can be answered in polynomial time.


%% file: introduction.tex
Argumentation is a widely adopted approach to non-monotonic reasoning with diverse applications in domains as varied as automated planning \citep{DBLP:journals/ijar/TezeGS22,DBLP:conf/nmr/HulstijnT04} and medical decision-making \citep{DBLP:conf/ecai/0002CV20,DBLP:conf/clima/FanCSTW13,DBLP:journals/iswa/CaropreseVZ22}. The notion of an argument underpins the approach, and argumentation theory allows one to determine which arguments are --- in some sense --- justified based on interactions between arguments, with such interactions most commonly taking the form of attacks. While \emph{structured} argumentation \citep{DBLP:journals/argcom/ModgilP14,DBLP:journals/argcom/Toni14,DBLP:conf/comma/YunVC20} concerns itself with the problem of how arguments are constructed and which arguments interact with others, \emph{abstract} argumentation treats arguments and inter-argument interactions as a given, and allows one to compute justified arguments according to specific \emph{argumentation semantics}.

The seminal work of \cite{dung_acceptability_1995} on abstract argumentation considered only arguments and binary attacks as inter-argument interactions. However, myriad extensions to this basic framework have been considered in the literature, including bipolar argumentation systems (which introduce the notion of support), argumentation systems with sets of attacking arguments \citep{DBLP:conf/ecsqaru/YunV21,DBLP:conf/argmas/NielsenP06}, weighted argumentation systems \citep{DBLP:conf/ecai/BistarelliS10,DBLP:conf/atal/DunneHMPW09}, and preference-based argumentation systems \citep{DBLP:conf/uai/AmgoudC98,DBLP:conf/comma/KaciTV18}, among others. The latter allows one to specify a preference ordering over pairs of arguments. These preferences are then used to determine whether an attack between two arguments is preserved, removed, or reversed according to a specific ``reduction'' process. An argumentation semantics over preference-based argumentation systems thus considers the resulting reduced argumentation framework rather than the original system to determine which sets of arguments are justified.
We call this process \textit{post-reduction inference with argumentation semantics} (PIAS).

There has been a recent wave of interest in the so-called inverse problems related to argumentation 
\citep{kido2022bayesian,e1bd0ada23b14b13a8b33b9869e604fb,mumford22complexity}. 
Inverse problems aim to determine unknown input parameters from observed outputs and known inputs. In this paper, we study such an inverse problem in the context of PIAS. The inputs consist of an argumentation framework, an argumentation semantics, and a labelling specifying the justification status of arguments. The objective is to identify the preferences between arguments that when combined with the given semantics yield the desired labelling for the argumentation framework. We note that Mahesar, et al. [\citeyear{DBLP:journals/corr/abs-2403-17653}] previously identified some necessary conditions for preferences over specific argument graph features (e.g., ``if argument $A$ attacks argument $B$ and both arguments are not defended by any arguments, then $A$ must be preferred to $B$ if $A$ is justified and $B$ is not'') under a reduction where preferences reverse arguments. However, they considered only one possible reduction with a relatively limited analysis of correctness, completeness and complexity. In contrast, we consider PIAS under four well-known reductions and directly target the question of when/whether a solution exists. The main contributions of this paper are as follows.

\begin{enumerate}
    \item A unified framework to discuss all four preference-based reductions and their associated inverse problems.
    \item A complexity analysis for the existence of solutions for those inverse problems.
\end{enumerate}

The inverse problem we consider allows one to infer a set of preferences given an argumentation framework, semantics and conclusions, effectively eliciting a set of preferences. Preference elicitation is an important problem across multiple domains, including economics \citep{10.1093/bmb/lds020}, public policy \citep{morgan2014use}, and recommender systems \citep{priyogi2019preference}. The ability to infer preferences from arguments opens up interesting possibilities and applications, which we discuss in Section \ref{Section:conclusions}.

The remainder of this paper is structured as follows. In the next section, we provide the necessary background around abstract and preference-based argumentation.
%
Section \ref{sec:pref-red1} forms the main contribution of this paper, i.e., we formally introduce the inverse problem and study the solutions for the four cases at hand.
Finally, we discuss the system and related work in Section \ref{Section: related works} before concluding  in Section \ref{Section:conclusions}.

%% file: background.tex
In abstract argumentation, arguments are atomic entities that interact through attacks. We recall the basic notions of abstract argumentation \citep{dung_acceptability_1995}.

\begin{definition}[AAF]
An (abstract) argumentation framework (AAF) is a pair $\AS = (\A, \C)$, where $\A$ is a set of arguments and $\C \subseteq \A \times \A$ is a set of attacks.
\end{definition}

We say $S \subseteq \A$ \emph{defends} $b$ iff  for all $c \in \A$ such that  $(c,b) \in \C$ there exists an $s \in S$ such that $(s,c) \in \C$.
The set of all \emph{direct attackers} of $a \in \A$ is denoted as $\Att(a) = \{ b\in \A \mid (b,a) \in \C\}$. 
A directed (resp. undirected) path from argument $a$ to $b$ (in $\AS$) is a sequence of arguments $(a_1, \dots, a_n)$ such that $a_1 = a, a_n = b$, and for all $ 1\leq i \leq n-1$, $ (a_i, a_{i+1}) \in \C$ (resp. $(a_i, a_{i+1}) \in \C$ or $(a_{i+1}, a_{i}) \in \C$).
%



Labellings \citep{DBLP:conf/jelia/Caminada06} are a popular approach to computing argumentation semantics. Here, arguments are labelled as  $\lin$, $\out$ or $\undec$ according to a fixed set of rules, reflecting their justification status. Most of the usual Dung \citeyear{dung_acceptability_1995} semantics comply with the standard approach to complete labellings, defined as follows. 

\begin{definition}[Complete labelling]\label{def:complete labelling}
Given $\AS = (\A,\C)$, a complete labelling (of $\A$) in $\AS$ is a function $l: \A \to \{\lin, \out, \undec\}$ 
such that for every $a \in \A$:
\begin{itemize}
    \item $l(a) = \lin $ iff for all $b \in \A,$ if $ (b,a) \in \C$ then $l(b) = \out$;
    \item $l(a) = \out$ iff  $\exists b \in \A$ such that $ (b,a) \in \C$ and $l(b) = \lin$;
    \item $l(a) = \undec$ otherwise.
\end{itemize}
\end{definition}

A labelling can also be denoted by a triple $(I, O, U)$, where $I, O, $ and $U$ are the set of arguments labelled $\lin$, $\out$, and $\undec$ respectively. When represented visually, we will color arguments in $I$, $O$, and $U$ with green, red, and gray respectively.
Given an AF $\AS$, we abbreviate
\[
        \AS_U=(U,\C\cap (U\times U)),
\]
for the subgraph obtained from $\AS$ by restriction to the $\undec$ labelled arguments, and similarly for $\AS_I$ and $\AS_O$.

\begin{example}\label{example:1}

Consider the argumentation framework $\AS = (\A, \C)$, where $\A = \{a,b,c,d\}$ and $\C  = \{ (a,b), (a,c),  (c,a), (b,c), (c,b),$ $(d,c), (c,d)\}$.
%
There are three complete extensions: $(\emptyset, \emptyset, \{a,b,c,d\})$, $(\{a,d\}, \{b,c\}, \emptyset)$, and $(\{c\}, \{a,b,d\}, \emptyset)$. The second of these is shown in Figure \ref{Fig:1a}.

\end{example}

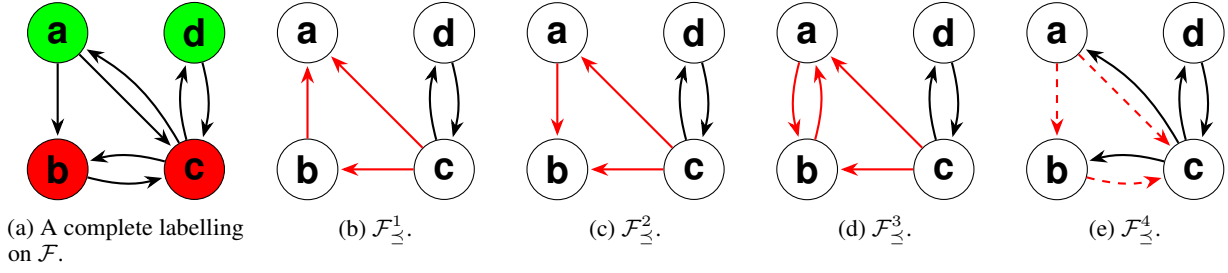
\begin{figure*}
   \centering
\begin{subfigure}[t]{0.19\textwidth}
\centering
\begin{tikzpicture}[
    scale=0.12,
    main node/.style={circle, draw, font=\sffamily\Large\bfseries, minimum size=0.8cm},
    edge style/.style={->, >=Stealth, thick, shorten >=1pt}
]

    \node[main node, fill=green] (a) {a};
    \node[main node, fill=red] (b) [below=of a] {b};

    \node[main node, fill=green] (d) [right=of a] {d};
    \node[main node, fill=red] (c) [below=of d] {c};

    \path[edge style]
        (a) edge (b)
        (a) edge (c)
        (b) edge [bend right=15] (c)
        (c) edge [bend right=15] (a)
        (c) edge [bend right=15] (b)
        (c) edge [bend left=15] (d)
        (d) edge [bend left=15] (c);

\end{tikzpicture}
\caption{A complete labelling on $\AS$.}
\label{Fig:1a}
\end{subfigure}\hfill
\begin{subfigure}[t]{0.19\textwidth}
\centering
\begin{tikzpicture}[
    scale=0.12,
    main node/.style={circle, draw, font=\sffamily\Large\bfseries, minimum size=0.8cm},
    edge style/.style={->, >=Stealth, thick, shorten >=1pt}
]

    \node[main node] (a) {a};
    \node[main node] (b) [below=of a] {b};

    \node[main node] (d) [right=of a] {d};
    \node[main node] (c) [below=of d] {c};

    \path[edge style]
        (b) edge [red] (a)
        (c) edge [red] (a)
        (c) edge [red] (b)
        (c) edge [bend left=15] (d)
        (d) edge [bend left=15] (c);

\end{tikzpicture}
\caption{$\AS^1_\preceq$.}
\label{fig:1b}
\end{subfigure}\hfill
\begin{subfigure}[t]{0.19\textwidth}
\centering
\begin{tikzpicture}[
    scale=0.12,
    main node/.style={circle, draw, font=\sffamily\Large\bfseries, minimum size=0.8cm},
    edge style/.style={->, >=Stealth, thick, shorten >=1pt}
]

    \node[main node] (a) {a};
    \node[main node] (b) [below=of a] {b};

    \node[main node] (d) [right=of a] {d};
    \node[main node] (c) [below=of d] {c};

    \path[edge style]
        (a) edge [red] (b)
        (c) edge [red] (a)
        (c) edge [red] (b)
        (c) edge [bend left=15] (d)
        (d) edge [bend left=15] (c);

\end{tikzpicture}
\caption{$\AS^2_\preceq$.}
\label{fig:1c}
\end{subfigure}\hfill
\begin{subfigure}[t]{0.19\textwidth}
\centering
\begin{tikzpicture}[
    scale=0.12,
    main node/.style={circle, draw, font=\sffamily\Large\bfseries, minimum size=0.8cm},
    edge style/.style={->, >=Stealth, thick, shorten >=1pt}
]

    \node[main node] (a) {a};
    \node[main node] (b) [below=of a] {b};
    \node[main node] (d) [right=of a] {d};
    \node[main node] (c) [below=of d] {c};

    \path[edge style]
        (a) edge [red, bend right=15] (b)
        (b) edge [red, bend right=15] (a)
        (c) edge [red] (a)
        (c) edge [red] (b)
        (c) edge [bend left=15] (d)
        (d) edge [bend left=15] (c);
\end{tikzpicture}
\caption{$\AS^3_\preceq$.}
\label{fig:1d}
\end{subfigure}\hfill
\begin{subfigure}[t]{0.19\textwidth}
\centering
\begin{tikzpicture}[
    scale=0.12,
    main node/.style={circle, draw, font=\sffamily\Large\bfseries, minimum size=0.8cm},
    edge style/.style={->, >=Stealth, thick, shorten >=1pt}
]

    \node[main node] (a) {a};
    \node[main node] (b) [below=of a] {b};

    \node[main node] (d) [right=of a] {d};
    \node[main node] (c) [below=of d] {c};

    \path[edge style]
        (a) edge [red, dashed] (b)
        (a) edge [red, dashed] (c)
        (b) edge [bend right=15,red, dashed] (c)
        (c) edge [bend right=15] (a)
        (c) edge [bend right=15] (b)
        (c) edge [bend left=15] (d)
        (d) edge [bend left=15] (c);

\end{tikzpicture}

\caption{$\AS^4_\preceq$.}
\label{fig:1e}
\end{subfigure}
\label{fig1}
\caption{An argumentation graph, a complete labelling, and its four reductions.}
\end{figure*}

Preference based argument frameworks extend AAFs to include preferences over arguments, representing the relative importance or desirability of arguments according to some (e.g., a reasoner's) viewpoint.
While such preferences are often modeled as partial orders \citep{DBLP:conf/uai/AmgoudC98,DBLP:conf/sum/AmgoudV10}, we restrict our attention to specific orderings to eliminate incomparability, lower computational complexity, produce clearer explanations, and better align our work with practical preference elicitation and dialogical reasoning. We note that this restriction has no practical effect on our results.

We recall that a total order $\preceq$ on a set elements $E$ is a reflexive, antisymmetric, transitive, and total relation over the set. For all $e_1, e_2 \in E$, $e_1 \preceq e_2$ means that $e_2$ is equal or more preferred than $e_1$. As usual, we will use the notation $e_1 \prec e_2$ iff $e_1 \preceq e_2$ and $e_2 \not \preceq e_1$; and $e_1 \simeq e_2$ iff $e_1 \preceq e_2$ and $e_2 \preceq e_1$.

Given a binary relation $R$, we write $\overline{R}$ for its converse, i.e.,  $\overline{R} = \{(b,a)\mid (a,b)\in R\}$.

When an argumentation framework decomposes into multiple connected components (CCs), deriving preferences between distinct components is not meaningful, since arguments in different components do not mutually influence each other and thus provide no information on preferences. Thus, we formalize the notion of CC-wise total orders.

\begin{definition}[Connected component]
Let $\AS = (\A, \C)$ be an abstract argumentation framework.
A set $S \subseteq \A$ is a \emph{connected component} (CC) of $\AS$ iff: (1)  $\forall a,b \in S$, there exists an undirected path from $a$ to $b$ and (2) $S$ is maximal with respect to set inclusion among the sets satisfying (1).
\end{definition}

\begin{definition}[CC-wise total order]
Given an argumentation framework $\AS = (\A,\C)$, a CC-wise total order $\preceq$ on $\AS$ is a partial order on $\A$ such that, for every connected component $S \subseteq \A$ of $\AS$, the restriction of $\preceq$ to $S$ is a total order on $S$.
\end{definition}

When an argumentation framework contains multiple disconnected components the number of possible preference orderings across the entire framework grows combinatorially. For example, in the AAF $(\{a,b,c,d\},\{(a,b),(c,d)\})$, if we fix $b \preceq a$ and $d \preceq c$, the remaining preferences between arguments from different components can take any ordering. From a computational perspective, this combinatorial explosion is unavoidable. 
We therefore focus instead on the set of possible preference orderings within each connected component, which is both more tractable and more meaningful, as preferences between disconnected components do not influence the evaluation of arguments. 

We thus treat a \emph{preference-based argumentation framework} \citep{DBLP:conf/uai/AmgoudC98} as a triple $\AS=(\A,\C,\preceq)$ where $(\A,\C)$ is an AAF, and $\preceq$ is CC-wise total order on $\AS$. 
Following the intuition of previous works \citep{DBLP:conf/comma/KaciTV18}, we note that a preference-based argumentation framework can be transformed into an abstract argumentation framework by transforming attacks into defeats. This transformation is called ``reduction'' and can follow different intuitions. Below, we describe four such reductions examined by Kaci et al. \citeyear{DBLP:conf/comma/KaciTV18}, noting where they were first proposed in the literature.


\begin{definition}\label{Def: reductions}
Given a preference-based argumentation framework $\AS = (\A,\C,\preceq)$, where $\preceq$ is CC-wise total order on $\AS$, we can obtain four abstract argumentation frameworks, called \emph{reductions}, by transforming attacks into \emph{defeats} using $\preceq$ on $\AS$:

\begin{itemize}
    \item $\AS^1_{\preceq} = (\A, \C_1)$, where $\C_1 = \{ (a,b) \mid  (a,b) \in \C, b \preceq a, \text{ or } (b,a) \in C, b \prec a \}$ \citep{amgoud2010role},
    \item $\AS^2_{\preceq} = (\A, \C_2)$, where $\C_2 = \{ (a,b) \mid  (a,b) \in \C, b \preceq a, or (a,b) \in C, (b,a) \not \in \C \}$ \citep{DBLP:conf/comma/KaciTV18}, and
    \item $\AS^3_{\preceq} = (\A, \C_1  \cup \C_2)$ \citep{DBLP:conf/comma/KaciTV18},
    \item $\AS^4_{\preceq} = (\A, \C_4)$, where $\C_4 = \{ (a,b) \in \C \mid b \preceq a\}$ \citep{amgoud2002inferring}.
\end{itemize}

\end{definition}

Once defeats have been computed, justified arguments are identified from the resulting argument framework using standard semantics (e.g., via computing complete labellings using the defeats as attacks).

\begin{example}[Cont'd]\label{example:2}
Consider the order $\preceq$ such that $  a\prec b \prec c \simeq d$. In figures \ref{fig:1b}, \ref{fig:1c}, \ref{fig:1d}, and \ref{fig:1e}, we represent the four AAFs originating from $\AS$ of Example \ref{example:1} using the four different reductions on $\preceq$. The solid red edges are used to indicate that a preference has been applied (resulting in a defeat), the dashed red edges act as a visual aid to mark the edges for Reduction 4 that are no longer present.
\label{ex:po}

\end{example}

%% file: contribution.tex
\label{sec:pref-red1}

As described above, each preference-based reduction treats preference orders over arguments differently. Rather than working directly with preferences over arguments, we introduce a unifying formalism in the form of \emph{preference functions}. These label attacks as $1$ if the source is preferred or equal to the target, and $0$ otherwise, subject to the constraint that the resulting attack structure remains free of inconsistent cycles.
We formalize this below.


 \begin{definition}[Preference function]\label{definition: preference function}
    A \emph{preference function} over an argumentation framework $\AS = (\A,\C)$ is a function $f\colon \C \to \{0,1\}$ such that any cycle in $\overline{F_0} \cup F_1$ is completely contained in $F_1\setminus \overline{F_0}$, i.e., the cycle is consistent, where
    \begin{align*}
        &F_0 = f^{-1}(\{0\}) \text{ and }F_1 = f^{-1}(\{1\}).
    \end{align*}
 \end{definition}

In the above definition if $(a,b)\in \overline{F_0}$, then $b\prec a$ and if $(a,b)\in F_1$ then $b\preceq a$. 
%
%
A cycle in $\overline{F_0} \cup F_1$ implies a cycle of preferences, while a cycle in $F_1 \backslash \overline{F_0}$ implies a cycle among equally preferred arguments.

Therefore, any cycle in $\overline{F_0} \cup F_1$ that is not contained in $F_1 \backslash \overline{F_0}$ contains one strict preference, i.e., it is of the following form $v_0 \prec v_1 \preceq v_2 \preceq \dots v_m \preceq v_0$, and our preferences are inconsistent.

\begin{example}[Cont'd]    
The function $f$ associated with $\preceq$ of Example \ref{example:2} assigns $0$ to $(a,b),(b,c)$, and $(a,c)$, and it assigns $1$ elsewhere. Notice that the red arrows in $\AS_\preceq^1$ represent the set $\overline{F_0}$, whereas the black ones represent $F_1\setminus \overline{F_0}$. The set $\overline{F_0}\cup F_1$ is therefore retrieved by taking all the arrows. One can identify inconsistent cycles as the ones that pass through at least one red arrow. Since the only cycle does not pass through red arrows, we know that $f$ is a preference function over $\AS$.

Now consider the three-cycle AAF $\mathcal{G}=(\{a,b,c\},\{(a,c),(c,b),(b,a)\})$. Assume that the associated function $f$ is one that assigs $0$ to all edges. There is then is a cycle in $\overline{F_0}\cup F_1$ and it is not contained in $F_1\setminus \overline{F_0}=\emptyset$. Thus, $f$ is not a preference function over $\AS$.

\end{example}

We recall that the inverse problem for PIAS (IPIAS): given an argumentation framework, an argumentation semantics, and a labelling specifying the status of arguments, determine the preferences between arguments that, when combined with the given semantics, yield the desired labelling.
We focus on the existence of a solution to IPIAS w.r.t.\ preference-based reductions under the complete semantics. Namely, we study the decision problem $\VPE_i$: 

\begin{quote}
``Given an argumentation framework $\AS = (\A,\C)$ and an arbitrary labelling $l$ on $\A$, is there a CC-wise total order $\preceq$ on $\AS$ such that $l$ is a complete labelling of the reduced framework $\AS^i_\preceq$?''
\end{quote}




\subsection{Preference-based Reduction 1: Attack reflection}


We begin by studying $\VPE_1$. In the following, we will be working mostly with preference functions. Therefore, we find it more practical to introduce the notion of an \textit{Inverting Preference Graph}, linking preference functions and reductions via Reduction 1 of Definition \ref{Def: reductions}. 

\begin{definition}[Inverting preference graph]\label{Def:IPG}
    An \emph{inverting preference graph (IPG)} of $\AS = (\A,\C)$ is an argumentation framework $\AS' = (\A,\C')$ such that $\C' = \overline{F_0} \cup F_1$ for some preference function $f$ over $\AS$.
\end{definition}

We now show that there is a correspondence between graphs obtained using Reduction 1 and IPGs.

\begin{proposition}\label{proposition:link1}
Let $\AS = (\A,\C)$ be an AAF. Any IPG of $\AS$ can be seen as a reduction $\AS_\preceq ^1$ of $\AS$ under some ordering $\preceq$ via Reduction 1, and conversely.
\end{proposition}

\begin{proof} 
If $\AS',\AS$, and $f$ are as in Definition \ref{Def:IPG}, the order given by
\begin{align*}
    &a\preceq_f b \iff (a,b)\in F_1\setminus \overline{F_0},\\
    &a\prec_f b \iff (a,b)\in \overline{F_0}
\end{align*}
is such that $\AS'=\AS^1_{\preceq_f}$. Indeed, $(a,b)$ is an attack in $\AS^1_{\preceq_f}$ if and only if $(a,b)\in \C$ and $b\preceq_f a$, in which case $f((a,b))=1$, or $(b,a)\in \C$ and $b\prec_f a$, in which case $f((b,a))=0$.

Conversely, let $\preceq$ be an order on $\AS$, and observe that the function
\[
    f_\preceq ((a,b))= \begin{cases}
        0 \hspace{2em} &\text{if }a\prec b,\\
        1 & \text{otherwise}
    \end{cases}
\]
is a preference function if $\preceq$ does not contain inconsistencies. The IPG $\AS'=(\A,\C')$ associated with $f_\preceq$ contains an attack $(a,b)\in \C'$ if and only if $f((b,a))=0$ or $f((a,b))=1$ if and only if there exists an attack $(b,a)\in \C$ and $b\prec a$ or there exists an attack $(a,b)$ and $b \preceq a$. This is exactly the definition of $\C_1$ in Definition \ref{Def: reductions}, thus $\AS'=\AS^1_\preceq$.
\end{proof}

\begin{proposition}\label{proposition: VPE2}
    The $\VPE_1$ problem for $(\AS,l)$ is answered positively if and only if the following three conditions hold:
        \begin{enumerate}
                \item $\lin$ labelled arguments can only be attacked and can only attack $\out$ labelled arguments, i.e.,
                \[
                        \big((I\times I) \cup (I\times U) \cup (U\times I)\big)\cap \C = \emptyset;
                \]
                \item every $\out$ labelled argument attacks or is attacked by an $\lin$ labelled argument, i.e.,
                \[
                        \forall a\in O \Big(\exists b \in I\big((a,b)\in \C \lor (b,a)\in \C\big)\Big);
                \]
                \item every connected component of $\AS_U$ contains a cycle.
        \end{enumerate}
\begin{proof}
        We begin by proving the only if statement. We suppose that at least one of the numbered conditions fails and we prove there is no solution for $\VPE_1$.

        Suppose that the first condition fails and that $(a,b)\in \C\cap(I\times I)$, the other cases are treated in the same manner. Any preference function $f$ on $\AS$ either assigns $1$ to $(a,b)$, in which case $(a,b)$ is in the associated IPG, or $f$ assigns $0$ to $(a,b)$, in which case $(b,a)$ is in the IPG. In both cases, the edge results in an attack between $\lin$ labelled arguments, against the completeness of $l$. Thus, $\VPE_1$ has no solutions.

        Similarly, if the second condition fails, suppose $a$ is an $\out$ labelled argument that witnesses it. That is, $a$ is labelled out and neither attacks nor is attacked by any $\lin$ labelled arguments. By construction, an IPG of $\AS$ does not introduce any new edges between two nodes that are not connected in $\AS$, thus $a$ cannot be attacked by $\lin$ labelled arguments in any IPG. Thus, $l$ is never complete and thus $\VPE_1$ has no solution.
        
         For condition three, observe that any connected component $V\subseteq U$ that does not contain a cycle must contain an element $v$ that has no attackers in $V$, i.e., an $\undec$ labelled argument without $\undec$ labelled attackers. Moreover, there is no way to add an attacker of $v$ that lives outside of $V$. Thus, an IPG should introduce a new cycle using preferences. However, such a cycle would be inconsistent. Thus, there is no IPG of $\AS$ making $l$ complete and $\VPE_1$ has no solution.

        We now move to the ``if statement'' and prove it directly by constructing an IPG, i.e., we produce a preference function $f$ on $\AS$.
        We start by defining the restriction of $f$ to $\AS_U$, which we call $f_U$. If $U=\emptyset$, there is nothing to define. Otherwise, consider $V\subseteq U$ a connected component of $\AS_U$. Identify a cycle in it, which we know exists by condition 3, and let $V_0\subseteq V$ be the set formed by the arguments that belong to the cycle. For all the other $n\in \mathbb{N}$, define
        \[
                V_{n+1}= \{v\in V\setminus V_n\mid \exists u \in V_n \left((v,u)\in \C \lor(u,v)\in \C\right)\}.
        \]
        These (finitely many non-empty) sets do not overlap and their union exhausts $V$, thus they yield a ranking on the arguments in $V$, namely $rk(u)=n$ iff $u\in V_n$. We define $f$ on the restriction of $\AS$ to $V$ by reversing all attacks from higher-ranked arguments to lower-ranked arguments, i.e., 
        \[
                f_V((u,v))=\begin{cases}
                        1 \hspace{2em} \text{ if }rk(u) \leq rk(v),\\
                        0 \hspace{2em} \text{ otherwise}.
                \end{cases}
        \]
        Now, all arguments in $V$ have an attacker in $V$. Indeed, if an argument has rank $0$ then it is in $V_0$, i.e., it is part of a cycle. Otherwise, it lives in a $V_{n+1}$ for some $n$, therefore there exists a $u\in V_n$ such that either $u$ attacks $v$ or $v$ attacks $u$. Either way, $f_V$ transforms it into an attack from $u$ to $v$.

        Observe also that no new loops are being introduced inside of $V$ by doing this. This is because the new attacks happen between two arguments of different rank, and no attacks from higher-ranked arguments to lower-ranked arguments are allowed, thus it would be impossible to loop back to the lower-ranked argument.

        We let $f_U$ be the union of the fucntions $f_V$ for all connected components $V$ of $\AS_U$. By condition 1, the remaining attacks in $\C$ involve at least one $\out$ labelled argument. For those, we define
        \[
                f_O((u,v))=\begin{cases}
                        1 \hspace{2em} l(v)=\text{out},\\
                        0 \hspace{2em} \text{otherwise}.
                \end{cases}
        \]
        By doing this, and given condition 2, we assure that every $\out$ labelled argument has an $\lin$ labelled attacker. Finally, we define
        \[
                f = f_O \cup f_U.
        \]
        By construction, $f$ is a function from $\C$ to $\{0,1\}$ and if it is a preference function, $l$ is complete on the IPG associated to $f$. We are left to check that $f$ is a preference function. We suppose there is a cycle,
        \[
                v_0\rightarrow v_1\rightarrow \dots \rightarrow v_n\rightarrow v_0,
        \]
        which is not completely contained in $F_1\setminus \overline{F_0}$ and look for a contradiction. If $v_i\notin O$, for all $0\leq i \leq n$, by condition 1 we get that they are all labelled $\undec$. Thus, they all live in the same connected component of $\AS_U$, and we already know these cycles will be contained in $F_1\setminus \overline{F_0}$ in the construction of $f_V$.
        In the other case, without loss of generality, assume $v_0$ is labelled $\undec$. Examining the definition of $f_O$, we get that $v_1\in O$, since the attack $(v_0,v_1)$ is in $F_1\cup \overline{F_0}$. Iterating this way of reasoning, we conclude that all the arguments in the cycle are labelled $\out$. Again by the definition of $f_O$, this implies that all the attacks in the cycle are assigned $1$ by $f$, therefore the cycle is contained in $F_1\setminus \overline{F_0}$. Contradiction. We conclude that $f$ is a preference function and $\VPE_1$ is answered positively.
\end{proof}
\end{proposition}

It follows that $\VPE_1$ is polynomial as it amounts to examining $I$, $O$, and $U$, as well as their cartesian products and intersections, and checking for the existence of a cycle in each connected component of $U$. 
Furthermore, following the construction in the proof, one can also provide a solution to the IPIAS$_1$ in polynomial time.

\subsection{Preference-based Reduction 2: Weak attack removal}

As we did for Reduction 1 in Definition \ref{Def:IPG}, we provide a definition that links a reduction via an order to a choice of preference function for the case of Reduction 2. For any argumentation framework $\AS = (\A,\C)$, we indicate with $\C_\leftrightarrow \subseteq \C$, the collection of its bidirectional attacks, i.e., 
\[
    \C_\leftrightarrow = \{(a,b)\in \C\mid (b,a)\in \C\}.
\]

\begin{definition}[Weakened preference sub-graph]\label{Def:WPG}
    A \emph{weakened preference graph (wPSG)} of $\AS = (\A,\C)$ is an inverting preference graph $\AS' = (\A,\C')$ such that its associated preference function $f$ satisfies, for all $a,b \in \C\setminus\C_\leftrightarrow$
    \[
        f((a,b))=1.
    \]
\end{definition}

\begin{proposition}\label{proposition:link2}
    Let $\AS = (\A,\C)$ be an AAF. Any wPSG of $\AS$ can be seen as a reduction $\AS_\preceq ^2$ of $\AS$ under some ordering $\preceq$ via Reduction 2, and conversely.
\end{proposition}

\begin{proof}
    Suppose $\AS'$ is a wPSG of $\AS$ and let $f$ be its associated preference function. Consider the order $\preceq_f$ as in Proposition \ref{proposition:link1} and let $\AS^2_{\preceq_f}$ be the reduction of $\AS$ over $\preceq_f$ via Reduction 2. We have that $(a,b)$ lives in $\AS^2_{\preceq_f}$ if and only if $(a,b)\in \C$ and $b\preceq_f a$ or if $(a,b)\in \C$ and $(b,a)\notin \C$. Both cases happen if and only if $f((a,b))=1$, thus $(a,b)\in \C'$, and we have $\AS'=\AS^2_{\preceq_f}$.

    Conversely, if $\preceq$ is an ordering on $\A$, let $f((a,b))=f_\preceq((a,b))$ for all $(a,b)\in \C_\leftrightarrow$, where $f_\preceq$ is the one in Proposition \ref{proposition:link1}, and $1$ elsewhere. Any unidirectional attack present in $\AS$ will be present in all of its reductions and, according to the definition of wPSG, in any of its wPSGs. At the same time, for a bidirectional attack $(a,b)\in\C_\leftrightarrow$ we have two cases. First, both attacks appear in $\C_2$ if and only if $b\simeq a$ if and only if $f$ assigns $1$ to both attacks. Second, $(a,b) \not \in \C_2$ if and only if $a \prec b$ if and only if $f$ assigns $0$ to $(a,b)$ and $1$ to $(b,a)$. That is to say, $\AS_\preceq^2=\AS'$.
\end{proof}

\begin{proposition}
    The $\VPE_2$ problem for $(\AS,l)$ is answered positively if and only if $l$ is already complete over $\AS$. 
    \begin{proof}
        Let $\AS=(\A,\C)$ be an arbitrary argumentation graph and suppose \(l\) is a complete labelling over $\AS$. Suppose $(a,b)\in\C$ is such that $(b,a)\notin \C$ and let $\C'=\C\cup\{(b,a)\}$. Then $l$ is complete over $\AS'=(\A,\C')$. Indeed, if $(a,b)$ is part of a complete labelling, it means that: $a$ is $\lin$ and $b$ is $\out$, or the contrary; or $a$ is $\undec$ and $b$ is $\out$, or the contrary; or both are $\undec$. In all five cases adding an  attack in the opposite direction will not invalidate the completeness of $l$.

        Since preferences in Reduction 2 can only turn bidirectional attacks into unidirectional attacks, we conclude that the existence of any wPSG making $\VPE_2$ true implies that $l$ was already complete on the original graph.
        
        Conversely, if $l$ is already complete, the trivial ordering where every two arguments are equivalent yields $\AS_\preceq^2=\AS$, thus $\VPE_2$ admits a solution.
    \end{proof}
\end{proposition}

Notice that this implies $\VPE_2$ for an arbitrary input $(\AS,l)$ can be answered in polynomial time as it amounts to checking that $l$ is complete over $\AS$ \citep{DBLP:books/sp/09/DunneW09}.

\subsection{Preference-based Reduction 3: Attack reflection and weak removal.}

The case for $\VPE_3$ is very similar to that of $\VPE_1$ but we can further relax the third condition. First, we need to link Reduction 3 to preference functions.

\begin{definition}[Combined preference graph]\label{def:CPG}
    A \emph{combined preference graph} (CPG) of $\AS = (\A,\C)$ is an AAF $\AS' = (\A,\C')$ such that $\C' = (\C\setminus \C_\leftrightarrow) \cup (\overline{F_0} \cup F_1)$ for some preference function $f$ over $\AS$.
\end{definition}

\begin{proposition}
    Let $\AS = (\A,\C)$ be an AAF, all CPGs of $\AS$ are reductions $\AS_\preceq ^3$ of $\AS$ under ordering $\preceq$ via Reduction 3.
\end{proposition}

\begin{proof}
     Let $\AS, \AS'$, and $f$ be as in Definition \ref{def:CPG}. Let $\preceq_f$ be as in Proposition \ref{proposition:link1} and consider $\AS^3_{\preceq_f}$. An attack $(a,b)$ lives in $\AS^3_{\preceq_f}$ if and only if one of the following three cases happens. First, $(a,b)\in \C$ and $b\preceq_f a$, which happens if and only if $f((a,b))=1$. Second, $(b,a)\in \C$ and $b\prec_f a$, which happens if and only if $f((b,a))=0$. Third, $(a,b)\in \C$ and $(b,a)\notin \C$, which happens if and only if $(a,b)\in \C\setminus \C_\leftrightarrow$. All three cases happen if and only if $(a,b)\in \C'$ of Definition \ref{def:CPG}, thus $\AS' = \AS^3_{\preceq_f}$.
     
     Conversely, let $\preceq$ be an order on $\A$. We let $f_\preceq$ be as in Proposition \ref{proposition:link1} and $\AS'$ its associated CPG. We have that $(a,b)$ lives in $\AS^3_{\preceq}$ if and only if it lives in one of $\AS^2_\preceq$ $\AS^1_\preceq$. We already know the first case happens if and only if $(a,b)\in \overline{F_0}\cup F_1$ (see Proposition \ref{proposition:link1}). For the second case, observe that if $(a,b)\in \C_\leftrightarrow$ we can resort to Proposition \ref{proposition:link2}, and that otherwise $(a,b)\in \C\setminus \C_\leftrightarrow\subseteq \C'$, thus $\AS'=\AS^3_\preceq$.
\end{proof}

\begin{proposition}\label{proposition: VPE3}
    The $\VPE_3$ problem for $(\AS,l)$ is answered positively if and only if the following three conditions hold:
        \begin{enumerate}
                \item $\lin$ labelled arguments only attack or are attacked by $\out$ labelled arguments
                \[
                        \big((I\times I) \cup (I\times U) \cup (U\times I)\big)\cap \C = \emptyset;
                \]
                \item every $\out$ labelled argument attacks or is attacked by an $\in$ labelled argument, i.e.,
                \[
                        \forall a\in O \Big(\exists b \in I\big((a,b)\in \C \lor (b,a)\in \C\big)\Big);
                \]
                \item every $\undec$ labelled argument attacks or is attacked by an $\undec$ labelled argument, i.e.,
                \[
                        \forall u\in U \Big(\exists v \in U\big((u,v)\in \C \lor (v,u)\in \C\big)\Big);
                \]
        \end{enumerate}
    \begin{proof}
    As in Reduction 1, if there is no attack (respectively, there is an attack) between the arguments $a$ and $b$ in $\AS$, then there is no attack (resp., there is an attack) between them in all CPGs of $\AS$. Thus, if even one of the three conditions fails, then there is no solution to $\VPE_3$.

    Assume the three conditions hold and consider a connected component $V$ of $\AS_U$. If all $v\in V$ already have an attacker in $V$, we let $f_V$ be $1$ for all attacks in $\C\cap (V\times V)$. Additionally, if $V$ contains a cycle, we define $f_V$ as in the proof of Proposition \ref{proposition: VPE2}. Otherwise, suppose $V$ does not contain a cycle. Notice that $V$ cannot contain bidirectional attacks. Pick an attack $(u,v)\in \C\cap (V\times V)$. Since $u\neq v$, by imposing $v \succ u$, Reduction 3 adds the attack $(u,v)$, thus we have a cycle in $V$ and we can construct $f_V$ as we did in the proof for Reduction 1. The same argument on ranks disproves the existence of inconsistent cycles.
    
    The rest of the proof is identical.
    \end{proof}
\end{proposition}

Observe that the third condition is equivalent to requiring that no $\undec$ labelled arguments be isolated in $\AS_U$.

This proof shows we can decide $\VPE_3$ in polynomial time. Furthermore, as for Reduction 1, we can provide a solution to the IPIAS$_3$ problem in polynomial time.

\subsection{Preference-based Reduction 4: Attack removal}

Similar to the previous sub-sections, we begin by connecting Reduction 4 to preference functions.

\begin{definition}[Filtering preference sub-graph]\label{def:PSG}
    A \emph{filtering preference sub-graph} (fPSG) of an AAF $\AS = (\A,\C)$ is an argumentation framework $\AS' = (\A,\C')$, where $\C' = F_1$ for some preference function $f$ over $\AS$.
\end{definition}

\begin{proposition}
    Let $\AS = (\A,\C)$ be an AAF, all fPSGs of $\AS$ are reductions $\AS_\preceq ^4$ of $\AS$ under ordering $\preceq$ via Reduction 4.
\end{proposition} 

\begin{proof}
    Let $\AS, \AS'$, and $f$ be as in Definition \ref{def:CPG}. Let $\preceq_f$ be as in Proposition \ref{proposition:link1} and consider $\AS^4_{\preceq_f}$. We have that $(a,b) \in \AS^4_{\preceq_f}$ if and only if $(a,b)\in \C$ and $b\preceq_f a$ if and only if $f((a,b))=1$. Thus $\AS'=\AS^4_{\preceq_f}$.
    
    The converse implication is identical choosing $f_\preceq$ to be the preference function associated to an arbitrary ordering $\preceq$ over $\A$.
\end{proof}

\input{reduction4}


\begin{proposition}
    Let $(\AS,l)$ be such that $\VPE_3$ and $\VPE_4$ are answered positively. Then $l$ is complete over $\AS$.
    \begin{proof}
        Recall the bullet points in Definition \ref{def:complete labelling}. If the first does not already hold for $l$, $\VPE_3$ is answered negatively by Proposition \ref{proposition: VPE3}. On the other hand, if the second bullet does not hold, then either an $\out$ labelled argument is not attacked by any $\lin$ arguments, in which case $\VPE_4$ is answered negatively (as  Reduction 4 can only remove attacks), or there is an argument which is labelled $\undec$ even though it is attacked by an $\lin$ labelled argument, in which case $\VPE_3$ is answered negatively. Lastly, if the first two conditions hold the third follows. We conclude that $l$ is complete over $\AS$.
    \end{proof}
\end{proposition}

We observe that the proof holds when we substitute $\VPE_1$ in place of $\VPE_3$, meaning that shared solutions across these reductions are only the ones in which $l$ is already a complete labelling for $\AS$. The converse implication of both cases follows trivially, since the order that sees all arguments as equivalent yields $\AS_\preceq^i=\AS$ for all $i=1,2,3,4$.

\begin{corollary}
    Let $S_i$, $1 \leq i \leq 4$, be the class of $(\AS,l)$ for which the problem $\VPE_i$ answers positively. It holds that $S_1\cap S_4 = S_3 \cap S_4 = S_2$, $S_1\subseteq S_3$.
\end{corollary}

Note that $S_i$, $1 \leq i \leq 4$, are distinct sets. Indeed, consider the argumentation graph $\AS=(\{a,b\},\{(a,b)\})$ and the three labellings $l_1(a)=\undec=l_1(b)$; $l_2(a)=\out$, $l_2(b)=\lin$; and $l_3(a)=\undec$, $l_3(b)=\lin$. Only $\VPE_3$ answers $(\AS,l_1)$ positively  and only $\VPE_4$ answers $(\AS,l_3)$ positively, from which we gather that $S_3\neq S_4$ and $S_1 \neq S_3$, thus $S_1\subsetneq S_3$ and $S_2\subsetneq S_3,S_4$. Since $(\AS,l_2)$ is answered positively by $\VPE_1$ and negatively by $\VPE_2$ and $\VPE_4$, we get $S_1\neq S_4$ and $S_2\subsetneq S_1$.

%% file: reduction4.tex
\begin{proposition}
\label{proposition: VPE4}

  The $\VPE_4$ problem for $(\AS,l)$ can be answered in polynomial time.
  \begin{proof}
    We divide this proof in three main steps. First, we prove we can safely ignore $\out$ labelled argument, as they can be handled in polynomial time and do not impact on the rest of the proof. Then, we introduce the decision problem $\RANK(\AS,l)$ and prove it is equivalent to $\VPE_4$ for $(\AS,l)$. Finally, we provide a polynomial time algorithm to solve $\RANK(\AS,l)$.

    \textbf{Step 1:} Observe that we can check in polynomial time (i.e., in $|\A| \cdot |\C|$ operations) that all $\out$ labelled arguments have an attack coming from an $\lin$ labelled argument. 
    Moreover, if an $\out$ labelled argument does not have an $\lin$ labelled attacker, we stop the search and answer the problem negatively.

    Suppose every $\out$ labelled argument has an $\lin$ labelled attacker. We can define the restriction of a preference function $f$ over $\AS$ to the attacks involving $\out$ labelled arguments as in the proof of Proposition \ref{proposition: VPE2}, i.e.,
    \[
      f_O ((u,v)) = \begin{cases}
	1 \hspace{2em} \text{if } l(v) = \out,\\
        0 \hspace{2em} \text{otherwise}.
      \end{cases}
    \]
    We observe that when searching for a preference function that answers $\VPE_4$, we can always assume it has this structure. Indeed, given any preference function $g$ over $\AS$ that makes $l$ complete, consider the function
    \[
      g_f ((u,v)) = \begin{cases}
	g((u,v)) \hspace{2em} \text{if } l(v) \neq \out \neq l(u),\\
	f_O((u,v)) \hspace{2em} \text{otherwise}.
      \end{cases}
    \]
    Since $g_f$ assigns $1$ to all attacks targeting $\out$ labelled arguments, any $\out$ labelled argument which is without an $\lin$ labelled attacker in the fPSG associated to $g_f$ is also without $\lin$ labelled attackers in the fPSG associated to $g$. Instead, if an argument $v$ has an $\lin$ labelled attacker $u$, either $v$ is labelled $\out$, in which case the attack is preserved, or $g_f((u,v))=g((u,v))$. Either way, the second condition of Definition \ref{def:complete labelling} is satisfied. The other two conditions of completeness are satisfied because $g_f$ agrees with $g$ on all the attacks whose source and target are not labelled $\out$. Thus, suppose $g_f$ is not a preference function over $\AS$, meaning there is an inconsistent cycle
    $
	    v_0 \to v_1 \to \dots \to v_n \to v_0
    $
    in $\overline{G_0} \cup G_1$, with $G_0=g_f^{-1}(\{0\})$ and $G_1=g_f^{-1}({1})$. Suppose that $v_0$ is labelled $\out$. Then $(v_0,v_1)\in G_1\setminus \overline{G_0}$, otherwise $(v_1,v_0)\in G_0$, meaning $g_f((v_1,v_0))=f_O((v_1,v_0))=0$, which is against the definition of $f_O$. Thus, $g_f((v_0,v_1))=1$ and we also get that $v_1$ is labelled $\out$. By iterating this piece of reasoning, we get that all arguments in the cycle are labelled $\out$ and all attacks are in $G_1\setminus G_0$, hence the cycle is consistent. Alternatively, none of the arguments $v_i$, for $0\leq i \leq n$, is labelled $\out$, and thus this cycle witnesses that $g$ is not a preference function over $\AS$, by the definition of $g_f$, leading to a contradiction. Therefore $g_f$ is a preference function over $\AS$, it has the desired form, and reducing $\AS$ by $g_f$ with Reduction 4 yields an argumentation framework for which $l$ is complete.

    \textbf{Step 2:} The construction above suggests that $\out$ labelled argument can be ignored in the rest of the proof, as by fixing $f_O$ we ensure that there are no cycles involving $\out$ labelled arguments and their condition for a complete labelling is already met. Hence, we assume $O=\emptyset$ going forward.

\begin{figure}
    \centering
 \begin{tikzpicture}[
    scale=0.08,
    main node/.style={circle, draw, font=\sffamily\Large\bfseries, minimum size=0.8cm},
    edge style/.style={->, >=Stealth, thick, shorten >=1pt}
]

    \node[main node, fill=green, label=above:{$0$}] (a) {a};
    \node[main node, fill=green, label=below:{$2$}] (f) [below=of a] {f};

    \node[main node, fill=green, label=above:{$1$}] (b) [right=of a] {b};
    \node[main node, fill=gray, label=below:{$2$}] (e) [below=of b] {e};

    \node[main node, fill=gray, label=above:{$2$}] (c) [right=of b] {c};
    \node[main node, fill=green, label=below:{$3$}] (d) [below=of c] {d};

    \path[edge style]
        (b) edge (a)
        (f) edge (b)
        (e) edge (b)
        (c) edge [bend right=15] (e)
        (d) edge (c)
        (e) edge [bend right=15] (c);
\end{tikzpicture}

\par\medskip 

\begin{tikzpicture}[
    scale=0.08,
    main node/.style={circle, draw, font=\sffamily\Large\bfseries, minimum size=0.8cm},
    edge style/.style={->, >=Stealth, thick, shorten >=1pt}
]

    \node[main node, fill=green] (a) {a};
    \node[main node, fill=green] (f) [below=of a] {f};

    \node[main node, fill=green] (b) [right=of a] {b};
    \node[main node, fill=gray] (e) [below=of b] {e};

    \node[main node, fill=gray] (c) [right=of b] {c};
    \node[main node, fill=green] (d) [below=of c] {d};

    \path[edge style]
        (b) edge (a)
        (f) edge (b)
        (e) edge (b)
        (c) edge [bend right=15] (e)
        (d) edge (c)
        (e) edge [bend right=15] (c)
        (b) edge (d);
\end{tikzpicture}

    \caption{Top: a labelled graph where $\RANK$ (Algorithm \ref{alg:rank}) answers \textbf{YES} and returns the ranking represented by the numbers next to the nodes. Bottom: a labelled graph for which $\RANK$ answers \textbf{NO}.}
    \label{fig:exp-RANK}
\end{figure}
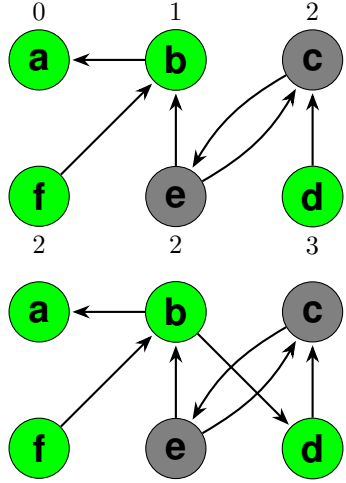

    We prove that $\VPE_4$ is equivalent to the following problem
    \begin{center}
      $\RANK(\AS,l)$: Given an argumentation framework $\AS=(\A,\C)$ and a labelling $l$ over $\AS$, decide whether there exists a function $\psi\colon \A \to \mathbb{Z}$, such that:
      \begin{enumerate}
	\item if $(u,v)\in \C$ and at least one of $u$ and $v$ is labelled $\lin$, then $\psi(u) > \psi(v)$; and 
	\item for every $\undec$ labelled argument $u$ we have that
          \[
	    \psi(u) \geq min\{\psi(v) \mid v\in \Att(u)\cap U\},
          \]
          where the minimum is $+\infty$ if the set is empty.
      \end{enumerate}
    \end{center}
    We call such a $\psi$ a \textit{ranking function}. An illustration of this decision problem is provided in Figure \ref{fig:exp-RANK}.
    
    Suppose we have a positive answer for $\RANK(\AS,l)$ and let $\psi\colon \A \to \mathbb{Z}$ be a witness to it. Let $f_\psi$ be a function such that it assigns $0$ to all attacks $(u,v)$ for which $\psi(u)>\psi(v)$ and $1$ otherwise. It follows that, for every $(u,v)\in \overline{F_0}\cup F_1$, we have $\psi(u)\leq \psi(v)$. This entails that the arguments in any given cycle in $\overline{F_0}\cup F_1$ all have the same value according to $\psi$. Thus all the attacks in the cycle belong to $F_1\setminus \overline{F_0}$.
    Therefore, $f_\psi$ is a preference function over $\AS$ and we need to prove that the fPSG of $\AS$ associated to $f_\psi$ sees $l$ as complete.
    We already know the second condition required for completeness is satisfied, since $O=\emptyset$ from the previous step. Let $a$ be an $\lin$ labelled argument. By the first condition on $\psi$, all attackers of $a$ will have a larger $\psi$-value than $a$ and all arguments attacked by $a$ will have a smaller $\psi$-value than $a$. Either way, the definition of $f_\psi$ will delete those attacks. Therefore, arguments labelled $\lin$ are isolated in the fPSG associated to $f_\psi$. Conversely, suppose that an argument $u$ has no attackers in the fPSG. If $u$ is labelled $\undec$, by the second condition of ranking function $\psi(u)=+\infty$, which is against $\psi$ having codomain $\mathbb{Z}$. Thus, $u$ is labelled $\lin$ and we conclude that $l$ is complete on the fPSG associated with $f_\psi$.
	          
    We now prove the converse  by induction on $n=|\A|$. %
    Let $n=1$ and observe that $\VPE_4$ can only be true if the argument is self-attacked and labelled $\undec$, or if it is isolated and labelled $\lin$. In both cases, the function that assigns $0$ to the argument is a ranking function.
      
    Suppose now that the induction hypothesis holds for all argumentation frameworks with at most $n$ arguments and that $\VPE_4$ for $(\AS,l)$ is answered positively. Let $f$ be the associated preference function to one of the fPSGs witnessing it. In the case that all arguments are labelled $\undec$, the function assigning $0$ to all arguments is a ranking function; the existence of a fPSG guarantees that each of them has an attacker. Otherwise, if there exists an $\lin$ labelled argument $c$ consider $\A_c=\A\setminus \{c\}$ and let $l_c$ and $\AS_c$ be the restriction of $l$ and $\AS$ to $\A_c$. Analogously, let $f_c$ be the restriction of $f$ to $\AS_c$. We argue that the fPSG $\AS'_c$ associated with $f_c$ is a witness to the fact that $\VPE_4$ for $(\AS_c,l_c)$ is answered positively. Indeed, $f_c$ must be a preference function, otherwise any inconsistent loop would invalidate $f$ being a preference function; all $\lin$ labelled arguments are isolated in $\AS'_c$, otherwise they would not be isolated in $\AS'$; and since we removed an $\lin$ labelled argument, attacks between $\undec$ labelled arguments are unaltered. Thus, by the inductive hypothesis on $(\AS_c,l_c)$, there exists a ranking function $\psi_c$ on $\A_c$. We consider the set of arguments that can reach $c$ in $\AS$, and that of arguments reachable by $c$ in $\AS$, namely
    \begin{align*}
        B_c = \{a \in \A' \mid \text{there is a path from } a \text { to } c \text{ in } \overline{F_0}\cup F_1\},\\
        A_c = \{a \in \A' \mid \text{there is a path from } c \text { to } a \text{ in } \overline{F_0}\cup F_1\}.
    \end{align*}
    Let $C_c= \A_c\setminus (A_c\cup B_c)$. Since $f$ is a preference function, $A_c\cap B_c=\emptyset$. Let $M = \max \{\psi_c(a)\mid a \in B_c\}$, $m = \min \{\psi_c(a) \mid a \in A_c\}$, and $n = \max (M+2-m, 0)$.
    Finally, we define
    \[
          \psi (a) = \begin{cases}
              \psi_c(a)+n &\text{if } a\in A_c,\\
              M+1 &\text{if } a=c,\\
              \psi_c(a) &\text{otherwise}.

          \end{cases}
    \]
    We check that $\psi$ is a ranking function. Let $(u,v)\in \C$ be such that at least one of $u$ and $v$ is labelled $\lin$, so $(v,u)\in\overline{F_0}$. If $u=c$, then $v\in B_c$, and we have $\psi(v)=\psi_c(v)<M+1=\psi(c)=\psi(u)$. Analogously, if $v=c$, we have $u\in A_c$ and $\psi(u)>M+1=\psi(c)=\psi(v)$. Observe that if $v\in A_c$, then $u$ must also live in $A_c$, since adding the attack $(v,u)\in \overline{F_0}$ to a path from $c$ to $v$ makes a path from $c$ to $u$. Analogously, if $u\in B_c$, also $v\in B_c$. In both of these cases, we know $(u,v)\in \C'_c$ and thus we get $\psi(u)>\psi(v)$ from $\psi_c(u)>\psi_c(v)$.
    Lastly, when $u\notin B_c$ and $v\notin A_c$, we have $\psi(u) \geq \psi_c(u) > \psi_c(v) = \psi(v)$.

    For condition 2 of the ranking function, let $u\in U$ be any $\undec$ labelled argument. We know there exists $v\in \Att(u)\cap U$ s.t.~$\psi_c(v)\leq \psi_c(u)$. If $v\in A_c$, then $u\in A_c$, and thus $\psi(v)=\psi_c(v)+n \leq \psi_c(u)+n = \psi(u)$. Otherwise, the rank of $v$ does not change and we get $\psi(v)=\psi_c(v)\leq\psi_c(u)\leq \psi(u)$. Thus $\psi$ satisfies the conditions to be a ranking function for $(\AS,l)$.

    \textbf{Step 3:} 
    We show that Algorithm \ref{alg:rank}, which instantiates $\RANK$, is polynomial.

\begin{algorithm}[t]
\caption{Computation of $\RANK$}
\label{alg:rank}
\begin{algorithmic}[1]

\State Enumerate the elements of $\A = \{u_1,\dots,u_n\}$
\State $\psi_0(u) \gets 0$ for all $u \in \A$; \label{line:init}
$k \gets 1$; 
$\psi_{k} \gets \psi_{0}$

\Repeat \label{line:repeat}

    \State $\psi_{k,1} \gets \psi_{k}$; 
    $i \gets 1$

    \While{$i \leq n$}

        \If{$u_i$ is labelled $\lin$}
            \State $\psi_{k,i}(u_i) \gets \max\big(\psi_{k}(u_i), \max\{\psi_{k,i}(v)+1 \mid v\in \A \text{ s.t. } (u_i,v)\in \C\}\big)$\label{line:increase in}
            \If{$\psi_{k,i}(u_i) > |\A|$}\label{line:in too big}
                 \Return \textbf{NO} 
            \EndIf

        \ElsIf{$u_i$ is labelled $\undec$}
            \If{$\Att(u_i)\cap U = \emptyset$}\label{line:undec unattacked}
                \Return \textbf{NO} 
            \EndIf

            \State $\psi_{k,i}(u_i) \gets$\label{line:increase undec}
\Statex \hspace{\algorithmicindent}
$\begin{aligned}[t]
  & \max\left(\psi_{k}(u_i),
    \min\{\psi_{k,i}(v)\mid v\in \Att(u_i)\cap U\},\right. \\
    & \left.\max\{\psi_{k,i}(v)+1\mid v\in I \text{ s.t. } (u_i,v)\in \C\}
\right)
\end{aligned}$

            \If{$\psi_{k,i}(u_i) > |\A|$}\label{line:undec too big}
                \Return \textbf{NO} 
            \EndIf
        \EndIf

        \State $\psi_{k,i+1} \gets \psi_{k,i}$; 
        $i \gets i+1$

    \EndWhile

    \State $\psi_{k+1} \gets \psi_{k,n}$
    \State $k \gets k+1$

\Until{$\psi_{k} = \psi_{k-1}$}\label{line:good condition}

\State \Return \textbf{YES},$\psi_k$\label{line:end}

\end{algorithmic}
\end{algorithm}

    Consider the \textbf{repeat-until} cycle. Observe that the rank assignments (lines \ref{line:increase in} and \ref{line:increase undec}) is monotonically non-decreasing, as the value appears in the evaluation of the $\max$. Thus, we only repeat if at least one argument's $\psi$ value has increased before line \ref{line:good condition}. By the Pigeonhole principle, since all arguments are assigned a starting  rank of $0$ (line \ref{line:init}), after $(\lvert \A\rvert+2)^2$ iterations of the \textbf{repeat-until} cycle one of them has a rank greater than $\lvert \A \rvert$, ensuring that the algorithm halts at lines \ref{line:in too big} or \ref{line:undec too big}.
    If we cache $\max$ and $\min$ values (making lines \ref{line:increase in} and \ref{line:increase undec} run in constant time) then complexity of the \textbf{while} loop is $O(|\A|)$ in the worst case. If the algorithm doesn't exit prematurely (i.e., with one of the returns in the \textbf{while} loop) then --- given that the \textbf{repeat-until} loop runs at most $O(|\A|^2)$ times. Thus, the worst case complexity of $\RANK$ is $O(|\A|^3)$.
    
    We now show that the algorithm solves $\RANK(\AS,l)$.

    Suppose $\RANK(\AS,l)$ is answered positively and that $\psi$ is a ranking function for $(\AS,l)$, with $m$ and $M$ being the minimum and maximum rank of the arguments according to $\psi$. Without loss of generality, we can assume $m=0$ and that 
    $M<\lvert A \rvert$. We prove inductively that at any step of the algorithm, it holds that $\psi_{k,i}(u)\leq \psi(u)$ for every argument $u$. Indeed, $\psi_1(u)=0\leq \psi(u)$ for all $u$, since $0$ is the minimum of $\psi$. If at any given step $k,i$ we have $\psi_{k,i}(u)\leq \psi(u)$ for all $u\neq u_i$ and $\psi_k(u_i)\leq \psi(u_i)$, then, assuming $u_i$ is labelled $\lin$, from $\psi(u_i) \geq \psi(v)+1$ for all $v$ such that $(u_i,v)\in\C$ we get $\psi_{k,i}(u_i)\leq \psi(u_i)$. The case for $u_i$ labelled $\undec$ is identical. Notice that since $\psi$ exists, all $\undec$ labelled arguments have an attacker, thus the algorithm will not halt at line \ref{line:undec unattacked}. Moreover, since $M< \lvert A \rvert$, the program does not halt at lines \ref{line:in too big} and \ref{line:undec too big}. We conclude that the program will halt at line \ref{line:good condition} with return \textbf{YES}.

    Conversely, suppose that the algorithm halts with \textbf{YES}, that is to say we exited at line \ref{line:end} and thus $\psi = \psi_{k,n} = \psi_{k,i} =\psi_{k-1, n}$ for some $k$, for all $1\leq i< n$. For every $\lin$ labelled argument $u_i$ and every $v\in \A$ such that $(u_i,v)\in \C$ we have that $\psi(u_i) = \psi_{k,i} (u_i) > \psi_{k-1,n}(v) = \psi(v)$; the same occurs if $u_i$ is labelled $\undec$ and $v$ is an $\lin$ labelled argument such that $(u_i,v)\in \C$. Finally, for every $\undec$ labelled argument $u_i$ there exists an argument $v\in \Att(u)\cap U$ such that $\psi(u_i) = \psi_{k,i}(u,i) \geq \psi_{k-1,n}(v) = \psi(v)$. Thus $\psi$ is equivalent to that in the problem statement of $\RANK(\AS,l)$.
  \end{proof}
\end{proposition}

%% file: related_work.tex
In argumentation theory, researchers tackle inverse problems in a variety of different settings. In the context of abstract argumentation for example, Dunne et al. [\citeyear{DBLP:journals/ai/DunneDLW15}] examined the problem of whether an argumentation framework exists that induces a given set of extensions with respect to some semantics. Mumford et al. [\citeyear{mumford22complexity}] considered the complexity of identifying defeat relations given a semantics, showing that doing so for 3-valued complete semantics can be done in polynomial time, while the problem is NP-complete for extension-based complete semantics. We refer the reader to the latter work for a discussion of other papers which consider similar problems over Dung-style argumentation frameworks. In the context of gradual semantics and ranked semantics, many researchers have considered inverse problems such as inferring attacks and initial weights given the outputs of such systems \citep{DBLP:journals/corr/abs-2203-01201,DBLP:conf/ijcai/OrenYVB22,DBLP:journals/argcom/OrenY23,DBLP:conf/comma/SkibaTRHK22}. 

Unlike the work above, we situate ourselves in the context of preference-based argumentation frameworks. Here, the only work we are aware of which attempts to tackle a similar problem is the one by Mahesar, Oren, and Vasconcelos [\citeyear{mahesar_computing_2018}] (noting that the same authors also considered a similar problem in within ABA \citep{DBLP:conf/prima/MahesarOV20}). However, the current work differs in several important respects. Mahesar et al.~sought to  identify \emph{all} (potentially exponential) solutions to the inverse problem; they considered only a single type of reduction; and they did not examine the complexity of determining whether the inverse problem can be answered positively.

Notably, Bernreiter, Dvo\v r\'ak, Rapberger, and Woltran [\citeyear{bernreiter2024effect}] work with the same reductions as we do and analyze the impact of preferences on the semantics of a narrower class of frameworks, namely well-formed Claim-Augmented Frameworks. Our work differs in the fact that we study an inverse problem (on AAFs) where preferences are reconstructed starting from a desired outcome (completeness of the labelling). As future work, using our work to inform the analysis of preference effects in claim-augmented settings may be interesting.

As noted by Mahesar, Oren, and Vasconcelos [\citeyear{mahesar_computing_2018,DBLP:conf/prima/MahesarOV20}] the problem we tackle here has several potential applications. The most obvious of these is preference elicitation --- given an argumentation framework presented by a user, the semantics they use for reasoning, and the conclusions they believe result, we can identify what preferences they operate under without having to explicitly query these.  Another potential application is in the area of sensitivity analysis, where we seek to identify how much (inferred) preferences must change before some conclusions change. Both of these applications are useful in contexts such as  recommender systems, decision support, and persuasion.

There are several obvious directions for future work. First, rather than focusing on existence of a solution, we can consider the problem of constructing a solution. While Mahesar et al. [\citeyear{mahesar_computing_2018}] has examined elements of this problem, their solution only applies to one reduction type, and is not complete. Furthermore, to overcome the potentially exponential number of solutions, we can seek to identify these solutions under a variety of criteria and constraints such as sparisty (minimising the number of strict preferences), robustness (maximising the number of preferences consistent with some outcome), or alignment with already elicited knowledge. Second, it would be interesting to extend the analysis to other semantics (e.g., preferred, stable, or CF2), and to other types of argumentation systems (e.g., weighted and probabilistic). Such extensions would allow us to understand which elements of our results arise directly due to the nature of the complete semantics, and which elements are more general. Another avenue of future work involves human studies. Here, we would seek to identify under which conditions specific reductions are used, and to determine how explanations around preferences and the reductions can be provided when explaining reasoning. Finally, since information is often noisy and inconsistencies may arise, examining approximate variants of the inverse problem (for example allowing a bounded number of labelling mismatches or attack modifications) and identifying parameterised complexity results could make our framework even more widely applicable and translate to learning-based argumentation settings.

%% file: conclusion.tex
We have considered an inverse problem for preference-based reductions in abstract argumentation. Namely, given an argumentation framework and a labelling, we examined the complexity of deciding whether there is a preference relation over arguments such that, after applying a given reduction, the resulting abstract framework admits the provided labelling as complete. To address this question across four well-known reductions, we introduced preference functions as a reduction-agnostic representation of which attacks are preserved or need to be processed (i.e., removed or reversed), together with global consistency constraints capturing the absence of forbidden preference cycles.

Our analysis shows that the computational difficulty of the inverse problem is polynomial-time solvability for all considered reductions, with a clean structural characterisation of feasible preferences. Beyond mere decidability, our framework also supports reasoning about the space of solutions, clarifying when preferences are essentially determined by the target labelling and when large families of compatible preference functions exist.